%% file: eremeev37.tex
\documentclass{llncs}
\input{preamble}

\newcommand{\plateau}{\text{\sc Plateau}\xspace} 
\newcommand{\bigO}[1]{\ensuremath{\mathcal{O}\left(#1\right)}}

\begin{document}
\mainmatter
\author{Anton V. Eremeev\inst{1}}
\title{On Non-Elitist Evolutionary Algorithms Optimizing Fitness Functions with a Plateau}
\titlerunning{Evolutionary Algorithms on Functions with a Plateau}
\institute{Sobolev Institute of Mathematics, Omsk, Russia\\
\email{eremeev@ofim.oscsbras.ru}
}

 \maketitle

\begin{abstract}
We consider the expected runtime of non-elitist evolutionary algorithms (EAs), when they are applied to a family of fitness functions $\plateau_r$ with a plateau of second-best fitness in a Hamming ball of radius $r$ around a unique global optimum. On one hand, using the level-based theorems, we obtain polynomial upper bounds on the expected runtime for some modes of non-elitist EA based on unbiased mutation and the bitwise mutation in particular. On the other hand, we show that the EA with fitness proportionate selection is  inefficient if the bitwise mutation is used with the standard settings of mutation probability.

\end{abstract}

\section{Introduction}
\label{sec:intro}

Realising the potential and usefulness of each operator that can constitute
evolutionary algorithms~(EAs) and their interplay is an important step
towards the efficient design of these algorithms for practical applications.
The proofs showing how and when the population size, recombination operators,
mutation operators or self-adaptation techniques are
essential in EAs can be found in~\cite{bib:Jansen2001,bib:Witt2008},
\cite{bib:Doerr2011}, \cite{bib:Lehre2013},
\cite{bib:Dang2016a} and other works. 

In the present paper, we study the efficiency of non-elitist EAs without recombination, applied to optimization problems with a single
plateau of constant values of objective function around the unique global optimum. Significance of plateaus analysis is associated with several reasons. Plateaus often occur in combinatorial optimization problems, especially in the unweighted problems, such as Maximum Satisfiability Problem~\cite{bib:HK96,bib:SHW10}. 
%
%
%
As a measure of efficiency, we consider the expected runtime, i.e., the expected number of objective (or fitness) function evaluations until the optimal solution is reached. We study the EAs without elite individuals, based on bitwise mutation, when they are applied to optimize fitness functions with plateaus of constant fitness.
To this end, we consider the $\plateau_r$ function with a plateau of second-best fitness in a ball of radius~$r$ around the unique optimum. The goal of this paper is to study the expected runtime of non-elitist EAs, optimizing $\plateau_r$, asymptotically for unbounded increasing number of binary variables~$n$, assuming constant parameter $r$.

It is shown in~\cite{bib:AntipovDoerr} that the (1 + 1) EA, which is one of the simplest mutation-based evolutionary algorithms, using an unbiased mutation operator (e.g., the bitwise mutation or the one-point mutation) optimizes $\plateau_r$ function with expected runtime $\frac{n^r (1+o(1))}{r! \Pr(1\le \xi \le r)},$ where $\xi$ is a random variable, equal to the number of bits flipped in an application of the mutation operator. 
This is proved under the condition that mutation flips exactly one bit with probability $\omega(n^{-\frac{1}{2r-2}})$. The most natural special case when this condition is satisfied is when exactly one bit is flipped with probability $\Omega(1)$.

In the present paper, with the similar conditions on unbiased mutation we obtain polynomial upper bounds on the expected runtime of non-elitist EAs, using tournament selection, $(\mu, \lambda)$-selection and, in the case of bitwise mutation with low mutation probability of order~$1/n^2$, using fitness proportionate selection. 
The bounds are obtained using the level-based theorems~\cite{bib:Corus2018}, \cite{bib:ErYuJOR17} and~\cite{bib:DK19}. 

Taking into account the similarity of function $\plateau_r$ to the well-known $\onemax$ function, we derive an exponential lower bound on the expected runtime of the EAs with the proportionate selection and standard mutation probability $1/n$, and more generally, with mutation probability $\chi/n$, where $\chi$ is a constant greater than $\ln{2}$. It is assumed that population size $\lambda=\Omega(n^{2+\delta})$ for some constant $\delta>0$.  In these conditions, we also show that finding an approximate solution within some constant approximation ratio also requires an exponential time in expectation. The lower bounds for the case of proportionate selection are based on the proof outlines suggested for linear functions in~\cite{bib:DEL19} and coincide with those results in the special case of $\onemax$ function. 

\section{Preliminaries}\label{sec:prelem}

We use the same notation as in~\cite{bib:Corus2018,bib:Dang2016,bib:Lehre2011}.
%
For any $n \in \Natural$, define $[n]:=\{1,2,\dots,n\}$.
The natural logarithm and logarithm to the base $2$ are denoted by
$\ln(\cdot)$ and $\log(\cdot)$ respectively.
For $x \in \{0,1\}^n$, we write $x_i$ for the $i$th bit value.
The Hamming distance is denoted by $H(\cdot,\cdot)$ and
the Iverson bracket by $[\cdot]$.
Throughout the paper the maximisation of a \emph{fitness function}
$f\colon \mathcal{X} \rightarrow \Real$ over a finite \emph{search
space} $\mathcal{X}:= \{0,1\}^n$ is considered.
Given a partition of $\mathcal{X}$ into $m$ ordered
subsets/\emph{levels} $(A_1,\dots,A_m)$,
let $A_{\geq j} := \cup_{i=j}^{m}A_i$.
Note that by this definition, $A_{\ge 1}={\mathcal X}$. 
A \emph{population} is a vector $P \in \mathcal{X}^{\lambda}$, where
the $i$th element $P(i)$ is called the $i$th \emph{individual}.
For $A \subseteq \mathcal{X}$, define $|P \cap A|:=|\{i \mid P(i)
\in A\}|$, \ie, the count of individuals of $P$ in $A$.

\subsection{The Objective Function}
We are specifically interested in two fitness functions defined on $\mathcal{X}
= \{0,1\}^n$:
\begin{itemize}
\item The most well-known benchmark function
$$
\onemax(x)  := \sum_{i=1}^{n} x_i,
$$
it is deeply studied in the literature on the theory of EAs, and will be referred here several times.
\item A function from~\cite{bib:AntipovDoerr} with a single plateau of the second-best fitness in a ball of radius~$r$ around the unique optimum
$$
\plateau_r  := \begin{cases}
         \onemax(x)    & \text{if } \onemax(x) \le n - r,\\
         n-r & \text{if } n - r < \onemax(x) < n,\\
         n  & \text{if } \onemax(x) = n,
        \end{cases}
$$
 parametrized by an integer $r$, assumed to be a constant greater than one.
\end{itemize}

Note that our results will also hold for the
generalised classes of such functions (see, e.g.,~\cite{Droste2002ANF}), where the meaning of $0$-bit and $1$-bit in
each position can be exchanged, and/or $x$ is rearranged according
to a fixed permutation before each evaluation.

\subsection{Non-Elitist Evolutionary Algorithm and Its Operators}
The non-elitist EAs considered in this paper fall into the framework of
Algorithm~\ref{algo:EA}, see, e.g.,~\cite{bib:Dang2016,bib:Lehre2011}. 
Suppose that the fitness function~$f(x)$ should be maximized.
Starting with some $P_0$
which is sampled uniformly from $\mathcal{X}^\lambda$, in each
iteration $t$ of the outer loop a new population $P_{t+1}$ is
generated by independently sampling $\lambda$ individuals from the
existing population $P_t$ using two
operators:
\emph{selection} $\sel\colon \mathcal{X}^\lambda \rightarrow [\lambda]$
and \emph{mutation} $\mut\colon \mathcal{X} \rightarrow \mathcal{X}$.
Here, $\sel$ takes a vector of $\lambda$ individuals as input, then
implicitly makes use of the function $f$, \ie, through \emph{fitness
evaluations}, to return the index of the individual to be selected.

\begin{algorithm}
  \caption{Non-Elitist Evolutionary Algorithm}\label{algo:EA}
  \begin{algorithmic}[1]
    \REQUIRE  Finite state space $\mathcal{X}$, and initial population $P_0 \in \mathcal{X}^\lambda$
    \FOR{$t=0,1,2,\dots$ until termination condition met}
        \FOR{$i=0,1,2,\dots,\lambda$ }
           \STATE Sample $I_t(i):=\sel(P_t)$, and set $x := P_t(I_t(i))$
           \STATE Sample $P_{t+1}(i) := \mut(x)$
        \ENDFOR
    \ENDFOR
  \end{algorithmic}
\end{algorithm}

The function is optimised when an optimum $x^*$, \ie, $f(x^*) =
\max_{x \in \mathcal{X}}\{f(x)\}$, appears in $P_t$ for the first
time, \ie, $x^*$ is sampled by $\mut$, and the optimisation time
(or runtime) is the number of fitness evaluations made until that
time. 

In this paper, we assume that the termination condition is never satisfied and 
the algorithm produces an infinite sequence of iterations. This simplifyng assumption is frequently used 
in the theoretical analysis of EAs.

Formally, $\sel$ is represented by a probability distribution over
$[\lambda]$, and we use $\psel(i \mid P)$ to denote the
probability of selecting the $i$th individual $P(i)$ of $P$. The well-known
\emph{fitness-proportionate selection} is an implementation of
$\sel$ with
$$
  \forall P \in \mathcal{X}^{\lambda}, \forall i \in [\lambda]\colon
    \psel(i \mid P) = \frac{f(P(i))}{\sum_{j=1}^{\lambda} f(P(j))}
$$
(if $\sum_{j=1}^{\lambda} f(P(j))=0,$ then one can assume that $\sel$ has the uniform distribution).
%
By definition, in the \emph{$k$-tournament selection}, $k$ individuals
are sampled uniformly at random with replacement from the
population, and a fittest of these individuals is returned. In
$(\mu,\lambda)$-\emph{selection}, parents are sampled uniformly at
random among the fittest $\mu$ individuals in the population. The
ties in terms of fitness function are resolved arbitrarily.

We say that $\sel$ is \emph{$f$-monotone} if for all $P \in
\mathcal{X}^\lambda$ and all $i,j \in [\lambda]$ it holds that
$\psel(i \mid P) \geq \psel(j \mid P) \Leftrightarrow f(P(i)) \geq
f(P(j))$. It is easy to see that all three selection mechanisms mentioned above are
$f$-monotone.


The \emph{cumulative selection probability} $\beta$ of $\sel(P)$
for any $\gamma \in (0,1]$ is
$$
  \beta(\gamma,P)
    := \sum_{i=1}^{\lambda} \psel(i\mid P) \cdot \left[ f(P(i)) \geq f_{\lceil \gamma\lambda \rceil}
    \right],\ \text{where} \ P \in \mathcal{X}^{\lambda},
$$
assuming a sorting $(f_1,\cdots,f_\lambda)$ of the fitnesses of $P$
in descending order. In essence, $\beta(\gamma, P)$ is the
probability of selecting an individual at least as good as the
$\lceil \gamma\lambda\rceil$-ranked individual of $P$,

When sampling $\lambda$ times with $\sel(P_t)$ and recording the outcomes as
vector $I_t \in [\lambda]^\lambda$, the \emph{reproductive rate} of $P_t(i)$ is
$$
  \alpha_t(i): = \expect{R_t(i) \mid P_t}
    \text{ where } R_t(i) := \sum_{j=1}^{\lambda}[I_t(j)=i].
$$
Thus $\alpha_t(i)$ is the expected number of times that
$P(i)$ is selected. The reproductive rate $\alpha_0$ of
Algorithm~\ref{algo:EA} 
is defined as $\alpha_0 := \sup_{t\geq 0} \max_{i\in[\lambda]}\{
\alpha_t(i) \}$.

The operator $\mut$ is represented by a transition matrix
${\pmut\colon \mathcal{X}\times \mathcal{X} \rightarrow [0,1]}$, and
we use $\pmut(y \mid x)$ to denote the probability to mutate an
individual $x$ into $y$.

In this paper, we consider the {\em unbiased} mutation operators~\cite{bib:LW12}.
This means that the probability distribution $\pmut(y \mid x)$
is invariant under bijection transformations of the Boolean cube~$\{0,1\}^n$, preserving the Hamming distance between
any pair of bitstrings $x$, $y$. This invariance may be regarded as invariance under
systematic flipping of arbitrary but fixed set of bit positions, and invariance under
systematically applying an arbitrary but fixed permutation to all the bits.

One of the most frequently used unbiased mutation operators,
the {\it bitwise mutation} (also known as the {\em standard bit mutation}), changes each
bit of a given solution with a fixed mutation probability $p_{mut}$.   Usually it is assumed that  $p_{mut}=\chi/n$ for some parameter $\chi>0$.
For the bitwise mutation with mutation probability $\chi/n$ we have
$$
  \forall x, y \in \{0,1\}^n\colon
  \pmut(y \mid x)
    = \left(\frac{\chi}{n}\right)^{H(x,y)}\left(1 - \frac{\chi}{n}\right)^{n - H(x,y)}.
$$
Another well-known mutation operator, the  {\em point mutation operator,} chooses $i$ randomly from $[n]$ and
changes only the $i$th bit in the given solution. 
Note that both of these  mutation operators treat the bit values $0$ and $1$
indifferently, as well as the bit positions, and therefore satisfy the conditions of unbiasedness. 


\section{Upper Bounds for Expected Runtime} \label{sec:ub_for_runtime}

\subsection{Tournament and $(\mu,\lambda)$-Selection} \label{subsec:ub_for_ranking}

First of all, due to similarity of function $\plateau_r$ with the well-known Jump function $\jump_r,$ analogously to the proof of Theorem~11 (its $\jump_r$ case) from~\cite{bib:Corus2018} we get

\begin{theorem}\label{thm:EA-ub_bitflip}
The EA applied to $\plateau_r,$ $r=\bigO{1},$ using 
\begin{itemize}
\item a bitwise mutation given a mutation rate $\chi/n$ for any fixed constant $\chi>0$,
\item $k$-tournament selection or  $(\mu,\lambda)$-selection with their parameters $k$ or $\lambda/\mu$ (respectively)
being set to no less than $(1+\delta)e^{\chi},$ where
$\delta\in (0,1]$ being any constant, and
\item population size $\lambda\ge c\ln{n},$ for a sufficiently large constant $c$
\end{itemize}
has the expected runtime $\bigO{n^r + n \lambda}$.
\end{theorem}

Note that by a slight modification of the proof of Theorem~11~\cite{bib:Corus2018},  one can also obtain the $\bigO{n^r + n \lambda}$ upper bound on the expected EA runtime in the case of Jump function.

In the general case of unbiased mutation we prove the following

\begin{theorem}\label{thm:EA-ub_unbiased}
The EA applied to $\plateau_r,$ $r=\bigO{1},$using 
\begin{itemize}
\item an unbiased mutation with $\Pr(\xi =0)\ge p_0=\Omega(1)$ and $\Pr(\xi =1) =\Omega(1),$ where $\xi$ is the random variable equal to the number of bits flipped in mutation,
\item $k$-tournament selection or  $(\mu,\lambda)$-selection with their parameters $k$ or $\lambda/\mu$ (respectively)
being set to no less than $(1+\delta)/p_0,$ where
$\delta\in (0,1]$ being any constant, and
\item population size $\lambda\ge c\ln{n}$ for sufficiently large constant $c$, independent of~$r$
\end{itemize}
has the expected runtime $\bigO{\lambda n^{r+1}}$.
\end{theorem}
\begin{proof}
Let us consider a partition of~$\mathcal X$ into $m=n-r+1$ subsets, where $A_i:=\{x: |x|= i-1\}, \ i\in [m-1], A_m:=\{x: |x|\ge n-r\}$.
Then from Theorem~\ref{thm:level-based-theorem}, analogously to the proof of Theorem 11 in~\cite{bib:Corus2018}, it follows that in conditions formulated above, 
on average after at most $C n$ iterations the EA will produce a population with at least $\gamma_0 \lambda$ individuals on the plateau $A_m$, where $C$ and $\gamma_0$ are positive constants. By the Markov's inequality, this implies that with probability at least~1/2, starting from any population, within $2Cn$ iterations, the EA produces a population with at least $\gamma_0 \lambda$ individuals on the plateau. 

Consider any iteration~$t$, when population~$P_t$ contains at least $\gamma_0 \lambda$ individuals on the plateau. For any offspring in the population $P_{t+1}$, the probability to have not less than $n-r+1$ ones is $\Omega(1/n).$
Given an individual $x$ with $n-r+i$ ones, the probability to produce $\mut(x)$ with at least $n-r+i+1$ ones is also $\Omega(1/n).$
By the EA outline, all individuals in each population $P_{t+i},$ $i=1,\dots,r$ are identically distributed, so, by the inductive argument, for any $i=1,\dots,r,$ the probability that the first individual produced in population $P_{t+i}$ will have at least $n-r+i$ ones is $\Omega(1/n^i).$ (Of course, an individual with any other index in population $P_{t}$ may be fixed here.) On the iteration~$t+r$, the individual number one is optimal with probability $\Omega(1/n^r).$

Now we can consider a sequence of series of the EA iterations, where the length of each series is~$2C n+r=\bigO{n}$ iterations. Suppose, $D_j,\ j=1,2,\dots,$ denotes an event of absence of the optimal individuals in the population throughout the $j$th series. In view of the above consideration, the probability of each event~$D_j, \ j=1,2,\dots,$ is~$1-\Omega(1/n^r),$ so
the probability to reach the optimum in at most~$j$ series is lower bounded by~$(1-C'/n^r)^j$ for some constant $C'$.

Let~$Y$ denote the random variable equal to the number of the
first series when the optimal solution is obtained. By the
properties of expectation (see, e.g.,~\cite{Gnedenko}),
$$
E[Y] = \sum_{j=0}^{\infty} {\Pr}(Y > j) = 1+\sum_{j=1}^{\infty}
{\Pr}(D_1\& \dots \& D_j) \le 1 +\sum_{j=1}^{\infty} (1-C'/n^r)^j = \bigO{n^r}.
$$
Consequently, the expected runtime is~$\bigO{\lambda n^{r+1}}$. \qed
\end{proof}

The requirement of a positive constant lower bound on probability to mutate none of the bits $\Pr(\xi =0)=\Omega(1)$ may be avoided at the expence of very high selection pressure and a factor of $\lambda$ longer runtime, using Theorem~\ref{thm:level-based-high_pressure}:

\begin{theorem}\label{thm:EA-ub_unbiased_high}
The EA applied to $\plateau_r,$ $r=\bigO{1},$ using an unbiased mutation with $\Pr(\xi =1) =\Omega(1),$ where $\xi$ is the random variable equal to the number of bits flipped in mutation,
\begin{itemize}
  \item and $k$-tournament selection, $k\ge n(1+\ln n)e/\Pr(\xi =1)$ with a population of size $\lambda \ge k$,
  \item or $(\mu,\lambda)$-selection with  $\lambda/\mu \geq n(1+\ln n)/\Pr(\xi =1)$
\end{itemize}
has the expected runtime $\bigO{\lambda^2 n^{r+1}}$.
\end{theorem}

The proof is analogous to that of Theorem~\ref{thm:EA-ub_unbiased}, but now the probability to choose a parent with the fitness $n-r$ within $n-r$ iterations of each series is lower-bounded only by $1/\lambda$, rather than by the constant~$\gamma_0$. So, the probability of each event $D_j$ is $1-\Omega(n^{-r}/\lambda)$ and $E[Y]=\bigO{n^{r} \lambda}$.

\subsection{Fitness-Proportionate Selection and Low Mutation Rate}\label{subsec:low_mut_fprop}

\begin{theorem}\label{thm:EA-fprop_ub}
The expected runtime of the EA on $\plateau_r,$ $r=\bigO{1},$
using
\begin{itemize}
\item fitness-proportionate selection,
\item bitwise mutation with mutation probability $\chi/n,$
$\chi=(1-c)/n$,
for any constant $c \in (0,1)$
\item population size
$\lambda \ge c' n^2 \ln(n),$  $\lambda={\cal O}\left(n^K\right),$
where  $c'$ and $K$ are positive sufficiently large constants
\end{itemize}
is
 $O(\lambda n^2 \log n + n^{2r+1}).$ 

\end{theorem}

\begin{proof}
We will apply Theorem~\ref{thm:level-based-theorem1} as in the proof of Theorem~5 from~\cite{bib:DEL19}. 
To this end, we use a partition of~$\mathcal X$ into $m=n+1$ subsets, where $A_i:=\{x: |x|= i-1\}, \ i\in [m-1], A_m:=\{x: |x|= n\}$.

Given $x\in A_{j}$ for any~$j<m-1$, among the first $j+1$ bits, there must be at least one $0$-bit,
thus it suffices to flip the first 0-bit on the left while keeping
all the other bits unchanged to produce a search point at a higher
level. The probability of such an event is
$ \frac{\chi}{n}\left(1 - \frac{\chi}{n}\right)^{n-1}
  >    \frac{\chi}{n}\left(1 - \frac{1}{n}\right)^{n-1}
  \geq \frac{1-c}{en^2} =: s_j, j \in [m-1].$ For $s_{m-1}$ we have $s_*:=s_{m-1}=\Omega((\frac{\chi}{n})^r).$ 
This choice of $s_j$ satisfies~(M1).
%
To satisfy (M2), we pick $p_0 := (1-\chi/n)^n$, \ie, the probability
of not flipping any bit position by mutation.

In (M3), we choose $\gamma_0 := c/4$ and for any $\gamma \leq
\gamma_0$, let $f_\gamma$ be the fitness of the $\lceil\gamma
\lambda\rceil$-ranked individual of any given $P \in
\mathcal{X}^\lambda$. Thus there are at least $k \geq
\lceil\gamma\lambda\rceil \geq \gamma\lambda$ individuals with
fitness at least $f_\gamma$ and let $s \geq k f_\gamma \geq
\gamma\lambda f_\gamma$ be their sum of fitness. We can pessimistically assume that
individuals with fitness less than $f_\gamma$ have fitness
$f_\gamma - 1$, therefore 
\begin{align*}
  \beta(\gamma, P)
    &\geq \frac{s}{s + (\lambda - k)(f_\gamma - 1)}
     \geq \frac{s}{s + (\lambda - \gamma\lambda)(f_\gamma - 1)} \\ 
    &\geq \frac{\gamma\lambda f_\gamma}{\gamma\lambda f_\gamma + (\lambda - \gamma\lambda)(f_\gamma - 1)}
     =    \frac{\gamma}{1 - (1 - \gamma)/f_\gamma} \\ 
    &\geq \frac{\gamma}{1 - (1 - c/4)/f^*}
     \geq \gamma e^{(1 - c/4)/f^*}, 
\end{align*}
where $f^* := n$ and in the last line we apply
the inequality $e^{-x}\ge 1-x.$
Note that $p_0 = (1-\chi/n)^{n} \geq e^{-\chi/(1-\varepsilon)}$
for any constant $\varepsilon \in (0,1)$ and sufficiently large
$n$. 
Indeed, by Taylor theorem, $e^{-z}=1-z+z\alpha(z),$ where
$\alpha(z)\to 0$ as $z\to 0$. So given any $\varepsilon>0$, for
all sufficiently small~$z>0$ holds $e^{-z}\le 1-(1-\varepsilon)z$.
For any~$\varepsilon\in(0,1)$ we can assume that
$z=\chi/(n(1-\varepsilon)),$ then for all sufficiently large~$n$
it holds that $(1-\chi/n)^{n}\ge
e^{-zn}=e^{-\chi/(1-\varepsilon)}.$
So we conclude that
\begin{equation*} 
  \beta(\gamma, P) p_0
    \ge \gamma e^{(1 - c/4)/f^*} e^{-\chi/(1 - \varepsilon)}
    \ge \gamma\left(1+\frac{1 - c/4 -\chi f^*/(1-\varepsilon)}{f^*}\right).
\end{equation*}
Since $\chi f^* \leq \chi n = 1 - c$, choosing
$\varepsilon := 1 - \frac{1-c}{1-c/2} \in (0,1)$
implies $\chi f^*/(1 - \varepsilon) \leq 1 - c/2$. Condition (M3)
then holds for $\delta:=c/(4 n)$ because
\begin{equation*} 
\beta(\gamma, P) p_0
  \geq \gamma\left(1+\frac{1 - c/4 - (1 - c/2)}{f^*}\right)
  \geq \gamma\left(1+\frac{c}{4 n}\right).
\end{equation*}

To verify condition~(M4'), we assume $C=1$ and note that
$$
               \frac{8}{\gamma_0\delta^2} \log\left(\frac{C m}{\delta} \left(\log \lambda +\frac{1}{\gamma_0 s_* \lambda}\right) \right)
\le  c''r n^2 \ln(n),
$$
for some constant $c''>0$, since $\lambda \le n^K$ and $s^*=\Omega(n^{-2r})$.
So~(M4') holds if $c'$ is large enough.
By Theorem~\ref{thm:level-based-theorem1},
we conclude that 
on average after at most $O(\lambda n^2 \log n + n^{2r+1})$ fitness evaluations 
the EA will produce the optimum. \qed
\end{proof}

In the case of $r=0$, the application of Theorem~\ref{thm:EA-fprop_ub} for $\lambda = \bigTheta{n^2 \log{n}}$ gives
$
  \expect{T} = \bigO{n^{4} \log^2{n}},
$
the same as the upper bound in Theorem~4.1 in~\cite{bib:DK19}.

\section{Inefficiency of Fitness-Proportionate Selection Given Standard Mutation Rate}\label{subsec:fprop_slow}

In this section, we consider Algorithm~\ref{algo:EA} with
fitness-proportionate selection and the bitwise mutation
given a constant value of the parameter~$\chi>\ln 2$. This
algorithm turns out to be inefficient on $\plateau_r$ for any constant~$r$. 
For the proof we will use the same approach as
suggested for lower bounding the EA runtime on the \onemax fitness
function in~\cite{bib:Lehre2011}. In order to obtain an upper
bound on the reproductive rate,
we first show that, roughly speaking, it is
unlikely that the average number of ones in the individuals of EA population becomes
less than $n/2$ sometime during an exponential number of iterations.

\begin{lemma}\label{lemma:antifitness_sum}
  Let $r$, $\varepsilon>0$ and $\delta>0$ be constants. 
  Define $T$ to be the smallest
  $t$ such that Algorithm~\ref{algo:EA}, applied to $\plateau_r$ function, with population size $\lambda\geq n^{2+\delta},$ using an $f$-monotone selection mechanism,  bitwise mutation with $\chi=\Omega(1),$ has a population  $P_t$ where
  $
  \sum_{j=1}^\lambda |P_t(j)|\leq \lambda (n/2)(1-\varepsilon).
  $
  Then there exists a constant $c>0$ such that ${\prob{T\leq e^{cn}} = e^{-\Omega(n^{\delta})}.}$
\end{lemma}

The proof of Lemma~\ref{lemma:antifitness_sum} is analogous to that of Lemma~9 from~\cite{bib:Lehre2010}. It is provided in the Appendix~B for the sake of completeness.

The main result of this section is Theorem~\ref{theorem:approximation} which   establishes an exponential lower bound for the expected time till finding an approximate solution with an approximation ratio $1+w$ for $w<(1-\ln(2)/\chi)$. The proof of this bound is based on the negative drift theorem for populations~\cite{bib:Lehre2010} (see
Theorem~\ref{th:negative_drift}) and Lemma~\ref{lemma:antifitness_sum}. Since this proof is rather technical, we start with a more straightforward lower bound for the EA runtime.

\begin{proposition}\label{prop:fprop_inefficiency}
 Let $\delta>0$  be a constant,
 then there exists a constant $c>0$ such that during $e^{cn}$ generations Algorithm~\ref{algo:EA} with population size~$\lambda\ge  n^{2+\delta},$ and $\lambda=\poly(n),$ bitwise mutation with mutation probability~$\chi/n$ for any constant $\chi>\ln{2}$, and
 fitness-proportionate selection, obtains the optimum of $\plateau_r(x)$  with probability at most~$e^{-\Omega(n^{\delta})}$.
\end{proposition}

The proof of Proposition~\ref{prop:fprop_inefficiency} is based on the same ideas as Corollary~13~\cite{bib:Lehre2011}.

\begin{proof}
It follows by Lemma~\ref{lemma:antifitness_sum} that with
probability at least~$1-e^{-\Omega(n^{\delta})}$ for any constant
$\varepsilon'\in (0,1)$ we have $\sum_{j=1}^{\lambda} \plateau_r(P_t(j)) \ge (1-\varepsilon')\lambda (n-r)/2 $ during $e^{c'n}$ iterations for some constant $c'>0$. Otherwise, with probability
$e^{-\Omega(n^{\delta})}$ we can pessimistically assume that the
optimum is found before iteration~$e^{c'n}$.

With probability at least~$1-e^{-\Omega(n^{\delta})}$ the
reproductive rate~$\alpha_0$ satisfies
\begin{equation}\label{eqn:alpha0ub1}
\alpha_0\le \frac{\lambda n}{(1-\varepsilon')\lambda (n-r)/2}\le \frac{2}{1-\varepsilon''},
\end{equation}
for some $\varepsilon''\in (0,1),$ assuming $n$ to be sufficiently large.

Inequality~(\ref{eqn:alpha0ub1}) implies that for a
sufficiently small~$\varepsilon''$ it holds that $\alpha_0<e^{\chi}$ and
using Corollary~\ref{cor:PK10}, we conclude
that the probability to optimise a function~$\plateau_r$ within $e^{c''n}$ generations is $\lambda e^{-\Omega(n)}$ for some constant $c''>0$. Therefore with
$c=\min\{c',c''\},$ the proposition follows.\  \qed
\end{proof}

The inapproximability result is established in

\begin{theorem}\label{theorem:approximation}
 Let $\delta>0$  be a constant,
 then there exists a constant $c>0$ such that during $e^{cn}$ generations
the EA with population size~$\lambda\ge
 n^{2+\delta},$ and $\lambda=\poly(n),$ bitwise mutation probability
 $\chi/n$ for any constant $\chi>\ln(2)$, and
 fitness-proportionate selection, applied to $\plateau_r$,
 will obtain a $(1-w)$-approximate solution, 
 $w<(1-\ln(2) / \chi)^2/2$, with probability at most~$e^{-\Omega(n^{\delta})}.$
\end{theorem}


\begin{proof}

As in the proof of Proposition~\ref{prop:fprop_inefficiency} we claim that with probability at least~$1-e^{-\Omega(n^{\delta})}$ the
reproductive rate~$\alpha_0$ satisfies the inequality 
$\alpha_0\le \frac{2}{1-\varepsilon''}$,
for any $\varepsilon''\in (0,1)$ assuming $n$ to be sufficiently large, 
and then the upper bound~$\alpha:=\frac{2}{1-\varepsilon''}$ satisfies condition~1 of
Theorem~\ref{th:negative_drift} for any $a(n)$ and $b(n)$. Note
that this $\alpha$ also satisfies the inequality
$\ln(\alpha)=\ln(2)-\ln(1-\varepsilon')<\ln(2)+\varepsilon' e$ for any $\varepsilon'\in(0,1/e)$.

Condition~2 of Theorem~\ref{th:negative_drift} requires that
${\ln(\alpha)/\chi+\delta'<1}$ for a constant ${\delta'>0}$. This
condition is satisfied because ${\frac{\ln(\alpha)}{\chi}<
\frac{\ln(2)+\varepsilon'e}{\chi}<1}$ for a sufficiently small
$\varepsilon'.$ Here we use the assumption that $\chi>\ln(2)$. It suffices to assume
$\varepsilon'=\frac{\chi-\ln 2}{2e}.$ Define
$\psi:=\frac{\ln(2)+\varepsilon'e}{\chi}=\frac{\ln(2)}{2\chi}+\frac{1}{2}.$

To ensure Condition~3 of Theorem~\ref{th:negative_drift}, we
denote $\rho:=\ln(2)/\chi<1$ and
 $$
 M(\chi):=\frac{1-\sqrt{\psi(2-\psi)}}{2}=
 \frac{1-\sqrt{\rho/2-\rho^2/4+3/4}}{2}.
 $$
 Note that
 $M(\chi)$ is decreasing in~$\rho$ and therefore increasing
 in~$\chi$, besides that
$M(\chi)$ is independent of~$n$ and $r$. Now we
define $a(n)$ and $b(n)$ so that $b(n)<M(\chi)n$ and
$b(n)-a(n)=\omega(n)$. Assume that $a(n):=n (1-\varepsilon)
M(\chi)$ and $b(n):=n(1-\varepsilon/2) M(\chi)$, where
$\varepsilon>0$ is a constant. 

Application of Theorem~\ref{th:negative_drift}
shows that with probability at most~$e^{-\Omega(n^{\delta})}$ the EA
obtains a search point with less than 
\begin{equation} \label{eqn:zerobits}
z(\varepsilon):=\frac{n(1-\varepsilon)}{2} \cdot
  \left(1-\sqrt{\frac{\rho}{2} - \left(\frac{\rho}{2}\right)^2 +\frac{3}{4}}\right)
\end{equation}
  zero-bits for any constant $\varepsilon\in(0,1)$.
Finally, using the Taylor series for the square root, we note that for any positive constant $w<(1-\rho)^2/2,$ there exists $\varepsilon\in(0,1),$ such that
the number of zero-bits in any $(1+w)$-approximate solution to $\plateau_r$ is at most~$z(\varepsilon)$.\  \qed
\end{proof}

\section{Discussion} \label{sec:discuss}

It is shown in~\cite{bib:AntipovDoerr} that under very general conditions we have mentioned in the introduction,  the (1+1)~EA easily (in expected $\bigO{n \log n}$ time) reaches the plateau and then performs a random walk on it, quickly approaching to a ``nearly-uniform'' distribution. A similar behaviour may be expected from the elitist EAs like $(\mu+\lambda)$~EA, where the best incumbent, once having reached the plateau, will travel on it, until the optimum is found. One can expect that in the case of non-elitist EAs, if the selection is strong enough, the population will stick to the plateau and spread on it as well. In the present paper, however, we have not identified such regimes yet. 

In the case of bitwise mutation, our Theorem~\ref{thm:EA-ub_bitflip} relies on a scenario, where the EA quickly reaches the edge of the plateau and most of the remaining time (with seldom possible retreats from the plateau) spends on the attempts to hit the optimum by ``large'' mutations, inverting up to $n-r$ zero-bits.
Theorem~\ref{thm:EA-ub_unbiased}, applicable to a wider class of mutation operators, relies on a more graduate scenario, where the search may consist of multiple stages, each one  starting from an ``arbitrary bad'' population, then reaches the edge of the plateau in expected $\bigO{n \log n}$ time and tries to hit the optimum by making $r$ sequential single-bit mutations, reducing the Hamming distance to the optimum by~1 in each EA iteration. If such an attempt fails, then we consider the next stage, over-pessimistically assuming that the search starts from a population of all-zero strings.  
Theorem~\ref{thm:EA-ub_unbiased_high} is even less demanding to the properties of mutation operators, but demanding very high selection pressure. It is likely that the runtime bound in this case may be significantly improved, since with such a high selection pressure the non-elitist EA becomes so close to the (1+1)~EA. 
While Theorems~\ref{thm:EA-ub_bitflip} -- \ref{thm:EA-ub_unbiased_high} deal with the tournament or $(\mu,\lambda)$-selection, 
Theorem~\ref{thm:EA-fprop_ub} shows that a similar situation may be observed in the case of fitness proportionate selection, although in this case we require that the mutation probability is reduced to~$\Theta(1/n^2)$ because otherwise the EA is likely to spend exponential time on the way to the plateau, as it is shown in Theorem~\ref{theorem:approximation}. 

\section{Conclusions}
This paper demonstrates the results, which are accessible by the available tools of runtime analysis. It also naturally leads to several questions for further research, some of which may require to develop principially new tools for EA analysis:
 \begin{itemize}
\item What are the leading constants in the obtained upper bounds?
\item What lower bounds can complement the obtained upper bounds?
\item Under what conditions on selection pressure is it possible to transfer the tight results on (1+1) EA from~\cite{bib:AntipovDoerr} to the non-elitist EAs?
\item How to extend the detailed runtime analysis to the Royal Road and Royal Staircase fitness functions (see, e.g.,~\cite{NC00}) which have multiple plateaus?
\item Would the genetic algorithms, which use the crossover operators, have any advantage over the mutation-based EAs considered in this paper?
 \end{itemize}
\section*{Acknowledgment}

The work was funded by program of fundamental scientific research of the Russian Academy of Sciences, I.5.1., project 0314-2019-0019. The author is grateful to Duc-Cuong Dang for helpful comments on preliminary version of the paper.

\section*{Appendix A}

This appendix contains the formulations of results employed from other works. Some of the formulations are given with slight modifications, which do not require a special proof.

Our lower bound is based on the \emph{negative drift theorem
  for populations}~\cite{bib:Lehre2010}.

\begin{theorem}\label{thm:negative-drift-pop}\label{th:negative_drift}
Consider the EA on $\mathcal{X} = \{0,1\}^n$ with
bitwise mutation rate $\chi/n$ and population size $\lambda =
\poly(n)$, let $a(n)$ and $b(n)$ be positive integers such that
$b(n)\leq n/\chi$ and $d(n) = b(n) - a(n) = \omega(\ln n)$. Given
$x^* \in \{0,1\}^n$, define $T(n) := \min\{t \mid |P_t \cap \{x
\in \mathcal{X} \mid H(x,x^*) \leq a(n)\}| > 0\}$. If there exist
constants $\alpha>1$, $\delta>0$ such that
  \begin{description}[noitemsep]
  \item[(1)] $\forall t\geq 0$, $\forall i \in [\lambda]\colon$
               if $a(n) < H(P_t(i),x^*) < b(n)$ then $\alpha_t(i) \leq \alpha$,
  \item[(2)] $\displaystyle
               \psi := \ln(\alpha)/\chi + \delta < 1$,
  \item[(3)] $\displaystyle
               b(n)/n < \min\left\{1/5,
                                           1/2 - \sqrt{\psi(2-\psi)/4}\right\}$,
  \end{description}
  then $\prob{T(n)\leq e^{cd(n)}} = e^{-\Omega(d(n))}$ for some constant $c>0$.
\end{theorem}

We also use a corollary of this theorem (Corollary~1 from~\cite{bib:Lehre2010}):

\begin{corollary} \label{cor:PK10}
The probability that a non-elitist EA with population size $\lambda=\poly(n),$ bitwise mutation probability $\chi/n,$ and maximal reproductive rate bounded by
 $\alpha< e^{\chi}-\delta$, for a constant $\delta > 0,$ optimises any function with a polynomial
number of optima within $e^{cn}$ generations is $e^{-\Omega(n)},$ for some constant $c > 0.$
\end{corollary}

To bound the expected optimisation time of Algorithm~\ref{algo:EA}
from above, we use the \emph{level-based analysis}~\cite{bib:Corus2018}. The following theorem is a re-formulation of 
Corollary~7 from \cite{bib:Corus2018}, tailored to the case of no recombination.
\begin{theorem}\label{thm:level-based-theorem}
Given a partition $(A_1,\dots,$ $A_{m})$ of $\mathcal{X}$, 
if 
there exist 
  $s_1,\dots,s_{m-1}, p_0,$ ${\delta \in(0,1]}$,
$\gamma_0 \in (0,1)$ such that
%
  \begin{description}[noitemsep]
  \item[(M1)] $\forall P\in\mathcal{X}^\lambda, \forall j\in [m-1] \colon$
    $\displaystyle
       \pmut\left( y\in A_{\geq j+1} \mid x \in A_j \right)\geq s_j,$
  \item[(M2)] $\forall P\in\mathcal{X}^\lambda, \forall j\in [m-1] \colon$
    $\displaystyle
     \pmut\left( y\in A_{\geq j} \mid x\in A_j \right)\geq p_0,$
  \item[(M3)] $\forall P \in \left(\mathcal{X}\setminus A_{m}\right)^\lambda,
               \forall \gamma\in(0,\gamma_0]\colon $
  $\displaystyle
               \beta(\gamma, P) \geq (1+\delta)\gamma/p_0,$ 
  \item[(M4)] population size 
  $\displaystyle \lambda \geq
               \frac{4}{\gamma_0\delta^2} \ln\left(\frac{128 m}{\gamma_0s_*\delta^2}\right),
               \text{ where } s_*:=\min_{j\in[m-1]} \{s_j\},
  $
  \end{description}
  then
  \begin{equation} \label{eqn:bound_CDEL18}
  \expect{T_0}
   <
    \left(\frac{8}{\delta^{2}}\right)
    \sum_{j=1}^{m-1}\left(
    \ln\left(\frac{6\delta\lambda}{4+\gamma_0
    s_j\delta\lambda}\right)+\frac{1}{\gamma_0 s_j \lambda}\right),
  \end{equation}
  where $T_0:=\min\{t \mid |P_t \cap A_{m}| \ge \gamma_0 \lambda\}.$
\end{theorem}

Note that literally the formulation of Corollary~2 in \cite{bib:Corus2018} gives the bound~(\ref{eqn:bound_CDEL18}) only for the expected runtime, but it is easy to see from the proof therein that the bound actually holds for the expected number $T_0$ of the first population that contains at least $\gamma_0 \lambda$ individuals in level $A_{m}$ as we put it in Theorem~\ref{thm:level-based-theorem}. This slight improvement is important in Section~\ref{sec:ub_for_runtime}.

As an alternative to Theorem~\ref{thm:level-based-theorem} we use the
new level-based theorem based on the {\em multiplicative up-drift}~\cite{bib:DK19}.
Theorem~3.2 from~\cite{bib:DK19} implies the following
\begin{theorem}\label{thm:level-based-theorem1}
Given a partition $(A_1,\dots,$ $A_{m})$ of $\mathcal{X}$,
define $T:=\min\{t\lambda \mid |P_t \cap A_{m}| > 0\}.$ 
If 
there exist 
  $s_1,\dots,s_{m-1},p_0,\delta \in(0,1]$,
%
$\gamma_0 \in (0,1)$, such that conditions (M1)--(M3) of Theorem~\ref{thm:level-based-theorem} hold and
%
  \begin{description}[noitemsep]
  \item[(M4')] for some constant $C>0$, the population size $\lambda$ satisfies
  $$
  \lambda \geq
               \frac{8}{\gamma_0\delta^2} \log\left(\frac{C m}{\delta} \left(\log \lambda +\frac{1}{\gamma_0 s_* \lambda}\right) \right),
               \text{ where } s_*:=\min_{j\in[m-1]} \{s_j\},
  $$
  \end{description}
  then
  $
  \expect{T}
   =\mathcal{O}
    \left(\frac{\lambda m \log(\gamma_0 \lambda)}{\delta} +
    \frac{1}{\delta}
    \sum_{j=1}^{m-1}\frac{1}{\gamma_0 s_j}\right). \ 
  $
\end{theorem}

Theorem~\ref{thm:level-based-theorem1} improves on Theorem~\ref{thm:level-based-theorem} in terms of dependence of the runtime bound denominator on $\delta$, but only gives an asymptotical bound. Its proof is analogous to that of Theorem~\ref{thm:level-based-theorem} and may be found in~\cite{bib:DEL19}. 

\begin{theorem}\label{thm:level-based-high_pressure}
Given an $f$-based partition $A_1,\ldots,A_{m}$ of~${\mathcal X}$, if the EA uses the mutation, such that ${\Pr}(\mut(x)\in A_{\ge j+1})\geq s_*$ for any $x\in A_{j},$ $j\in [m-1]$ 
  \begin{itemize}
  \item and a $k$-tournament selection, $k\ge \frac{(1+\ln m)e}{s_*}$ with a population of size $\lambda \ge k$,
  \item or $(\mu,\lambda)$-selection and $\lambda \geq \frac{\mu (1+\ln m)}{s_* }$
  \end{itemize} 
then an element from $A_{m}$ is found in expectation after at most $em$ genetations.
 \end{theorem}
          
{\bf The proof} is analogous to that of the main result in~\cite{bib:ErYuJOR17}.\\

\section*{Appendix B}

This appendix contains the proofs provided for the sake of completeness.

{\bf Proof of Lemma~\ref{lemma:antifitness_sum}. }

  For the initial population, it follows by a Chernoff bound that
  $\prob{T=1}=e^{-\Omega(n)}$.  We then claim that for all $t\geq 0$,
  $\prob{T=t+1\mid T>t}\leq e^{-c'n}$ for a constant $c'>0,$ which by
  the union bound implies that $\prob{T<e^{cn}}\leq
  e^{cn-c'n}=e^{-\Omega(n)}$ for any constant $c<c'$.

In the initial population, the expected number of ones of a $k$-th
individual, $k\in [\lambda]$ is
$
|P_0(k))|\le n/2.
$
It will be more convenient here to consider the number of zeros, rather than the number of ones. 
We denote $Z_t^{(j)}:= n - |P_t(j)|$ , for $t\geq
  0,$ $j\in[\lambda],$ and
$Z_t:=\lambda n - \sum_{j=1}^{\lambda} |P_t(j)|$. Let $p_j$ be
the probability of selecting  the $j$-th individual when producing the population in generation
  $t+1$.
  For $f$-monotone selection mechanisms, it holds that
  $\sum_{j=1}^\lambda p_j Z_t^{(j)} \leq Z_t/\lambda.$

Let $P=(x_1,\dots,x_{\lambda})$ be any deterministic population.
Denote the $i$-th bit of the $k$-th individual in $P$ by
$x^{(k,i)}.$ Denote $z_k:= n - |x^{(k)}|$,
$1\le k \le \lambda$, and $Z(P):=\sum_{k=1}^{\lambda} z_k$.  

Let us consider the bitwise mutation first.
The expected number of zero-bits in any offspring 
$j\in[\lambda]$ produced from population $P_t=P$ is
$$
 \expect{|P_{t+1}(j)|\ | \ P_t=P}= \sum_{k=1}^{\lambda} p_k\left[ \sum_{i=1}^n \left (x^{(k,i)} (1-\chi/n) + (1-x^{(k,i)}) \chi/n \right) \right],
$$
so the expected value of $Z_{t+1}^{(j)}$ for any offspring
$j\in[\lambda]$ is
  \begin{align*}
& \expect{Z_{t+1}^{(j)} \ | \ P_t=P}\le n- \expect{|P_{t+1}(j)|\ | \ P_t=P}\\
&= \sum_{k=1}^{\lambda} p_k\left[ \sum_{i=1}^n \left (x^{(k,i)} (1-\chi/n) + (1-x^{(k,i)}) \chi/n \right) \right]\\
&=n-\sum_{k=1}^{\lambda} p_k \left[ \chi+ (1-2\chi/n) \sum_{i=1}^n x^{(k,i)}  \right]\\
&\le n-\chi- (1-2\chi/n) \sum_{k=1}^{\lambda} p_k  (n- z_k)\\
&=\chi+ (1-2\chi/n) \sum_{k=1}^{\lambda} p_k z_k \le \chi + (1-2\chi/n)Z(P)/\lambda.
  \end{align*}

If $T>t$ and $Z(P)<\lambda n(1+\varepsilon)/2$, then
  \begin{align*}
    \expect{Z_{t+1}\mid P_t=P}
     &\le \lambda \chi  + Z(P)\left(1-2\chi/n\right)\\
     & < \lambda \chi + \frac{\lambda n}{2}(1+\varepsilon)\left(1-2\chi/n\right)
       = \frac{\lambda n}{2}(1+\varepsilon)-\varepsilon\lambda\chi.
  \end{align*}
  Now $Z_{t+1}^{(1)}, Z_{t+1}^{(2)}, \dots,
  Z_{t+1}^{(\lambda)}$ are non-negative independent random variables,
  each bounded from above by~$n$,
  so using the Hoeffding's inequality~\cite{Hoeffding63}
  we obtain
  \begin{align*}
    \prob{ Z_{t+1} \geq \frac{\lambda n}{2}(1+\varepsilon)}
   &  \leq   \prob{ Z_{t+1} \geq \expect{Z_{t+1}} + \varepsilon\lambda\chi}\\
   & \leq \exp\left( - \frac{2(\varepsilon\lambda\chi)^2}{\lambda n^2}\right) ,
  \end{align*}
which is $e^{-\Omega(n^{\delta})}$ since $\lambda\ge n^{2+\delta}.$  
\   \qed

\bibliographystyle{splncs03}
\bibliography{references}

%
\end{document}

%% file: preamble.tex
\usepackage{amsfonts,amsmath}
\usepackage{xspace}                           
\usepackage{booktabs}                         
\usepackage{graphics}
\usepackage{tikz}
\usepackage{dsfont}

\allowdisplaybreaks

\newcommand{\prob}[1]{\Pr\left(#1\right)}
\newcommand{\expect}[1]{\mathbf{E}\left[#1\right]}

\newcommand{\bigTheta}[1]{\mathord{\Theta}\mathord{\left(#1\right)}}

\newcommand{\Real}{\mathbb{R}}
\newcommand{\Natural}{\mathbb{N}}

\newcommand{\sel}{\text{\tt select}\xspace} 
\newcommand{\mut}{\text{\tt mutate}\xspace}

\newcommand{\psel}{\ensuremath{p_\mathrm{sel}}\xspace}
\newcommand{\pmut}{\ensuremath{p_\mathrm{mut}}\xspace}

\DeclareMathOperator{\poly}{poly}

\usepackage{mathrsfs}

\newcommand{\ab}{\hspace{0.125em}}                        
\newcommand{\ie}{\hbox{i.\ab e.}\xspace}                  

\usepackage[noend]{algorithm,algorithmic}
\floatname{algorithm}{Algorithm}

\newcommand{\onemax}{\text{\sc OneMax}\xspace} 

\newcommand{\jump}{\text{\sc Jump}\xspace} 

\usepackage{enumitem}

\usepackage{color}

